\documentclass[11pt]{article}

\usepackage{algorithm}

\makeatletter
\renewcommand{\fnum@algorithm}{\fname@algorithm}
\makeatother

\usepackage{mathrsfs}
\usepackage{enumitem}
\usepackage{bbm}
\usepackage{tikz}
\usepackage{float}
\usetikzlibrary{backgrounds}
\usepackage{color}
\usepackage{graphicx}
\usepackage{latexsym}
\usepackage{amsfonts}
\usepackage{pifont,xspace,fullpage,epsfig, wrapfig}
\usepackage{amsmath, amssymb, amsthm} 
\usepackage{multirow}
\usepackage{dsfont}
\usepackage{array}

\usepackage{adjustbox}
\usepackage{caption}
\DeclareSymbolFont{AMSb}{U}{msb}{m}{n}
\DeclareMathSymbol{\N}{\mathbin}{AMSb}{"4E}
\DeclareMathSymbol{\Z}{\mathbin}{AMSb}{"5A}
\DeclareMathSymbol{\R}{\mathbin}{AMSb}{"52}
\DeclareMathSymbol{\Q}{\mathbin}{AMSb}{"51}
\DeclareMathSymbol{\erert}{\mathbin}{AMSb}{"50}
\DeclareMathSymbol{\I}{\mathbin}{AMSb}{"49}
\DeclareMathSymbol{\C}{\mathbin}{AMSb}{"43}

\newcommand{\mynote}[2]{{\textcolor{#1}{ #2}}}

\definecolor{gray}{gray}{0.4}
\newcommand{\gray}[1]{\mynote{gray}{{\footnotesize #1}}}

\newcommand{\remove}[1]{}

\newtheorem{theorem}{Theorem}[section]
\newtheorem{lemma}[theorem]{Lemma}
\newtheorem{definition}[theorem]{Definition}
\newtheorem{remark}[theorem]{Remark}
\newtheorem{proposition}[theorem]{Proposition}
\newtheorem{claim}[theorem]{Claim}

\newtheorem{observation}[theorem]{Observation}

\newtheorem{example}[theorem]{Example}

\newcommand{\1}{\mathbbm{1}}

\newcommand{\AAA}{\mathcal A}

\newcommand{\DDD}{\mathcal D}

\newcommand{\eps}{\varepsilon}

\newcommand{\error}{{\rm error}}

\newcommand{\Lap}{\operatorname{\rm Lap}}

\newcommand{\VC}{\operatorname{\rm VC}}

\newcommand{\poly}{\mathop{\rm poly}}

\newcommand{\halfplane}{\operatorname*{\tt HALFPLANE}\nolimits}

\newcommand{\kgon}{\operatorname*{\tt k-GON}\nolimits}
\newcommand{\kuniongon}{\operatorname*{\tt k-UNION-GON}\nolimits}
\newcommand{\kigon}{\operatorname*{\tt k_1-GON}\nolimits}
\newcommand{\kmgon}{\operatorname*{\tt k_m-GON}\nolimits}
\newcommand{\triang}{\operatorname*{\tt TRIANGLE}\nolimits}
\newcommand{\interior}{\operatorname*{\tt interior}}
\newcommand{\ckgon}{\operatorname*{\tt CONVEX-k-GON}\nolimits}

\def\Q{\operatorname*{\mathbb{Q}}}
\def\poly{\mathop{\rm{poly}}\nolimits}
\def\Lap{\mathop{\rm{Lap}}\nolimits}

\def\disj{\mathop{\rm{DISJ}}\nolimits}
\def\conj{\mathop{\rm{CONJ}}\nolimits}

\makeatletter
\newcommand{\thickhline}{%
    \noalign {\ifnum 0=`}\fi \hrule height 1pt
    \futurelet \reserved@a \@xhline
}
\newcolumntype{"}{@{\hskip\tabcolsep\vrule width 1pt\hskip\tabcolsep}}
\makeatother

\makeatletter
\newlength{\fboxhsep}
\newlength{\fboxvsep}
\newlength{\fboxtoprule}
\newlength{\fboxbottomrule}
\newlength{\fboxleftrule}
\newlength{\fboxrightrule}
\def\@frameb@xother#1{%
  \@tempdima\fboxtoprule
  \advance\@tempdima\fboxvsep
  \advance\@tempdima\dp\@tempboxa
  \hbox{%
    \lower\@tempdima\hbox{%
      \vbox{%
        \hrule\@height\fboxtoprule
        \hbox{%
          \vrule\@width\fboxleftrule
          #1%
          \vbox{%
            \vskip\fboxvsep
            \box\@tempboxa
            \vskip\fboxvsep}%
          #1%
          \vrule\@width\fboxrightrule}%
        \hrule\@height\fboxbottomrule}%
    }%
  }%
}
\long\def\fboxother#1{%
  \leavevmode
  \setbox\@tempboxa\hbox{%
    \color@begingroup
    \kern\fboxhsep{#1}\kern\fboxhsep
    \color@endgroup}%
  \@frameb@xother\relax}

\makeatother

\begin{document}

\title{Differentially Private Learning of Geometric Concepts}
\author{
Haim Kaplan\thanks{Tel Aviv University and Google.}
\and
Yishay Mansour\footnotemark[1]
\and
Yossi Matias\thanks{Google.}
\and
Uri Stemmer\thanks{Ben-Gurion University. Supported by a gift from Google Ltd.}
}

\date{February 13, 2019}
\maketitle

\begin{abstract}
We present differentially private efficient algorithms for learning union of polygons in the plane (which are not necessarily convex). Our algorithms achieve $(\alpha,\beta)$-PAC learning and $(\eps,\delta)$-differential privacy using a sample of size $\tilde{O}\left(\frac{1}{\alpha\eps}k\log d\right)$, where the domain is $[d]\times[d]$ and $k$ is the number of edges in the union of polygons.
\end{abstract}

\section{Introduction}

Machine learning algorithms have exciting and wide-range potential. However, as the data frequently contain sensitive personal information, there are real privacy concerns associated with the development and the deployment of this technology. 
Motivated by this observation, the line of work on {\em differentially private learning} (initiated by \cite{KLNRS08}) aims to construct learning algorithms that provide strong (mathematically proven) privacy protections for the training data. Both government agencies and industrial companies have realized the importance of introducing strong privacy protection to statistical and machine learning tasks. A few recent examples include Google~\cite{erlingsson2014rappor} and Apple~\cite{thakurta2017learning} that are already using differentially private estimation algorithms that feed into machine learning algorithms, and the US Census Bureau announcement that they will use differentially private data publication techniques in the next decennial census~\cite{USCensus}. Differential privacy is increasingly accepted as a standard for rigorous privacy. We refer the reader to the excellent surveys in~\cite{DR14} and~\cite{VadhanTutorial2016}. The definition of differential privacy is,

\begin{definition}[\cite{DMNS06}]
Let $\AAA$ be a randomized algorithm whose input is a sample. Algorithm $\AAA$ is $(\eps,\delta)$-{\em differentially private} if for every two samples $S,S'$ that differ in one example, and for any event $T$, we have 
$$\Pr[\AAA(S)\in T]\leq e^{\eps}\cdot \Pr[\AAA(S')\in T]+\delta.$$
\end{definition}

For now, we can think of a (non-private) learner as an algorithm that operates on a set of classified random examples, and outputs a hypothesis $h$ that misclassifies fresh examples with probability at most (say) $\tfrac{1}{10}$. A {\em private learner} must achieve the same goal while guaranteeing that the choice of $h$ preserves differential privacy of the sample points. Intuitively this means that the choice of $h$ should not be significantly affected by any particular sample.

While many learning tasks of interest are compatible with differential privacy, privacy comes with a cost (in terms of computational resources and the amount of data needed), and it is important to understand how efficient can private learning be. Indeed, there has been a significant amount of work aimed at understanding the sample complexity of private learning~\cite{BKN10,CH11,BNS13,BNS13b,FX14,BNSV15,AlonLiMaMo18}, the computational complexity of private learning~\cite{BunZ16}, and studying variations of the private learning model~\cite{BNS15,BNS16,DworkF18,BassilyTT18}. However, in spite of the significant progress made in recent years, much remains unknown and answers to fundamental questions are still missing. In particular, the literature lacks effective constructions of private learners for specific concept classes of interest, such as halfspaces, polynomials, $d$-dimensional balls, and more. We remark that, in principle, every (non-private) learner that works in the {\em statistical queries (SQ)} model of Kearns~\cite{Kearns98} can be transformed to preserve differential privacy. However, as the transformation is only tight up to polynomial factors, and as SQ learners are often much less efficient than their PAC learners counterparts, the resulting private learners are typically far from practical.

In this work we make an important step towards bringing differentially private learning closer to practice, and construct an effective algorithm for privately learning simple geometric shapes, focusing on polygons in the plane. To motivate our work, consider the task of 
analyzing GPS navigation data, or the task of learning the shape of a flood or a fire based on users' location reports. 
As user location data might be sensitive, the ability to {\em privately learn} such shapes 
 is of significant importance.

\subsection{A Non-Private Learner for Conjunctions and Existing Techniques}
Our learner is obtained by designing a private variant for the classical (non-private) learner for conjunctions using the greedy algorithm for set-cover. Before describing our new learner, we first quickly recall this non-private technique (see e.g.,~\cite{Kearns94} for more details).

Let $\conj_{k,d}$ denote the class of all conjunctions (i.e., {\rm AND}) of at most $k$ literals over $d$ Boolean variables $v_1,\dots,v_d$, e.g., $v_1\wedge\overline{v}_4\wedge v_5$. Here, for a labeled example $(\vec{x},\sigma)$, the vector $\vec{x}\in\{0,1\}^d$ is interpreted as an assignment to the $d$ Boolean variables, and $\sigma=1$ iff this assignment satisfies the target concept.
Given a sample $S=\{(\vec{x}_i,\sigma_i)\}$ of labeled examples, the classical (non-private) learner for this class begins with the hypothesis $h=v_1\wedge\overline{v}_1\wedge\dots\wedge v_d\wedge\overline{v}_d$, and then proceeds by deleting from $h$ any literal that ``contradicts'' a positively labeled example in $S$. Observe that at the end of this process, the set of literals appearing in $h$ contains the set of literals appearing in the target concept (because a literal is only deleted when it is contradicted by a positively labeled example).

The next step is to eliminate unnecessary literals from the hypothesis $h$ (in order to guarantee generalization). Note that removing literals from $h$ might cause it to err on negative example in $S$, and hence, the algorithm must carefully choose which of the literals to eliminate. This can be done using the greedy algorithm for set cover as follows. We have already made sure that each of the literals in $h$ does not err on positive examples in $S$ (since such literals were deleted), and we know that there is a choice of $k$ literals from $h$ that together correctly classify all negative examples in $S$ (since we know that the $k$ literals of the target concept are contained in $h$). Thus, our task can be restated as identifying a small number of literals from $h$ that together correctly classify all negative examples in $S$. This can be done using the greedy algorithm for set cover, where every literal in $h$ corresponds to a set, and this set ``covers'' a negative example if the literal is zero on this example.

To summarize, the algorithm first identifies the collection of all literals that are consistent with the positive data, and then uses the greedy algorithm for set cover in order to identify a small subset of these literals that together correctly classify the negative data. This is a good starting point for designing a private learner for conjunctions, since the greedy algorithm for set cover has a private variant~\cite{GuptaLMRT10}. 
The challenge here is that in the private algorithm of Gupta et al.~\cite{GuptaLMRT10}, the collection of sets from which the cover is chosen is assumed to be fixed and independent of the input data. In contrast, in our case the collection of sets corresponds to the literals that correctly classified the positive data, which is data dependent. One might try to overcome this challenge by first identifying, in a private manner, a collection $L$ of all literals that correctly classify (most of) the positive data, and then to run the private algorithm of Gupta et al.~\cite{GuptaLMRT10} to choose a small number of literals from $L$. However, a direct implementation of this idea would require accounting for the privacy loss incurred due to each literal in $L$. As $|L|$ can be linear in $d$ (the number of possible literals), this would result in an algorithm with sample complexity $\poly(d)$. When we apply this strategy to learn polygons in the plane, $d$ will correspond to the size of an assumed grid on the plane, which we think of as being very big, e.g., $d=2^{64}$. Hence, $\poly(d)$ sample complexity is unacceptable.

\subsection{Our Results}

Our first result is an efficient private learner for conjunctions. Our learner is obtained by modifying the strategy outlined above to use the greedy algorithm for set cover in order to choose a small number of literals directly out of the set of all possible $2d$ literals (instead of choosing them out of the set of literals that agree with the positive examples). However, this must be done carefully, as unlike before, we need to ensure that the selected literals will not cause our hypothesis to err on the positive examples. Specifically, in every step of the greedy algorithm we will aim at choosing a literal that eliminates (i.e., evaluates to zero on) a lot of the negative examples without eliminating (essentially) any of the positive examples. In the terminology of the set cover problem, we will have {\em two} types of elements -- positive and negative elements -- and our goal will be to identify a small cover of the negative elements that does not cover (essentially) any of the positive elements. We show,

\begin{theorem}
There exists an efficient $(\eps,\delta)$-differentially private $(\alpha,\beta)$-PAC learner for $\conj_{k,d}$ with sample complexity\footnote{For simplicity we used the $\tilde{O}$ notation to hide logarithmic factors in $\alpha,\beta,\epsilon,\delta,k$. The dependency in these factors wil be made explicit in the following sections.} $\tilde{O}(\frac{1}{\alpha\eps}k\log d)$.
\end{theorem}
We remark that our techniques extend to disjunctions of literals in a straightforward way, as follows.
\begin{theorem}
There exists an efficient $(\eps,\delta)$-differentially private $(\alpha,\beta)$-PAC learner for
the class of all disjunctions (i.e., OR) of at most $k$ literals over $d$ Boolean variables with sample complexity $\tilde{O}(\frac{1}{\alpha\eps}k\log d)$.
\end{theorem}

We then show that our technique can be used to privately learn (not necessarily convex) polygons in the plane. 
To see the connection, let us first consider {\em convex} polygons, and observe that a {\em convex} polygon with $k$ edges can be represented as the intersection of $k$ halfplanes. Thus, as in our learner for Boolean conjunctions, if we could identify a halfplane that eliminates (i.e., evaluates to zero on) a lot of the negative examples without eliminating (essentially) any of the positive examples in the sample, then we could learn convex polygons in iterations. However, recall that in our learner for conjunctions, the parameter $d$ controlled both the description length of examples (since each example specified an assignment to $d$ variables) and the number of possible literals (which was $2d$). Thus, the running time of our algorithm was allowed to be linear in the number of possible literals. The main challenge when applying this technique to (convex) polygons is that the number of possible halfplanes (which will correspond to the parameter $d$) is huge, and our algorithm cannot run in time linear in the number of possible halfplanes.

To recover from this difficulty, we consider the dual plane in which sample points correspond to lines, and show that in that dual plane it is possible to (privately) identify a point (that corresponds to a halfplane in the primal plane) with the required properties (that is, eliminating a lot of negative examples while not eliminating positive examples). The idea is that since there are only $n$ input points, in the dual plane there will be only $n$ lines to consider, which partition the dual plane into at most $n^2$ regions. Now, two different halfplanes in the primal plane correspond to two different points in the dual plane, and if these two points fall in the same region, then the two halfplanes behave identically on the input sample. We could therefore partition the halfplanes (in the primal plane) into at most $n^2$ equivalence classes w.r.t.\ the way they behave on the input sample. This fact can be leveraged in order to efficiently implement the algorithm.

Our techniques extend to non-convex polygons, which unlike convex polygons cannot be represented as intersection of halfplanes. It it well known that every (simple\footnote{A {\em simple} polygon is one which does not intersect itself.}) polygon with $k$ edges can be represented as the union of at most $k$ triangles, each of which can be represented as the intersection of at most 3 halfplanes (as a triangle is a convex polygon with 3 edges). In other words, a (simple) polygon with $k$ edges can be represented as a DNF formula (i.e., disjunction of conjunctive clauses) in which each clause has at most 3 literals. As we will see, our techniques can be extended to capture this case efficiently. Our main theorem is the following.

\begin{theorem}[informal]
There exists an efficient $(\eps,\delta)$-differentially private $(\alpha,\beta)$-PAC learner for union of (simple) polygons in the plane with sample complexity $\tilde{O}(\frac{1}{\alpha\eps}k\log d)$, where $k$ is the number of edges in the union of polygons and $\log d$ is the description length of examples.
\end{theorem}

As the greedy algorithm for set cover has many applications in computational learning theory, we hope that our techniques will continue to find much broader use.

\section{Preliminaries}\label{sec:preliminaries}

We recall standard definitions from learning theory and differential privacy. In the following $X$ is some arbitrary domain. A concept (similarly, hypothesis) over domain $X$ is a predicate defined over $X$. A concept class (similarly, hypothesis class) is a set of concepts. 

\begin{definition}[Generalization Error]
Let $\DDD\in \Delta(X)$ be a probability distribution over $X$ and let $c:x\rightarrow\{0,1\}$ be a concept. 
The {\em generalization error} of a hypothesis $h:X\rightarrow\{0,1\}$ w.r.t.\ $\DDD$ and $c$ is  
defined as $\error_{\DDD}(c,h)=\Pr_{x \sim \DDD}[h(x)\neq c(x)]$.
\end{definition}

We now recall the notion of PAC learning~\cite{Valiant84}.
Let $C$ and $H$ be a concept class and a hypothesis class over a domain $X$, and let $\AAA:\left(X\times\{0,1\}\right)^n\rightarrow H$ be an algorithm that operates on a labeled database and returns a hypothesis from $H$.

\begin{definition}[PAC Learner~\cite{Valiant84}]\label{def:PAC}
Algorithm $\AAA$ is an $(\alpha,\beta)$-PAC learner for concept class $C$ using hypothesis class $H$ with sample complexity $n$ if for every distribution $\DDD$ over $X$ and for every fixture of $c\in C$, given a labeled database $S=\left(\left(x_i,c(x_i)\right)\right)_{i=1}^n$ where each $x_i$ is drawn i.i.d.\ from $\DDD$, algorithm $\AAA$ outputs a hypothesis $h\in H$ satisfying
$$\Pr\left[ \error_{\DDD}(c,h) > \alpha\right] \leq \beta.$$
The probability is taken over the random choice of
the examples in $S$ according to $\DDD$ and the coin tosses of the learner $\AAA$.
\end{definition}

Without privacy considerations, the sample complexity of PAC learning is essentially characterized by a combinatorial quantity called the \emph{Vapnik-Chervonenkis (VC) dimension}:

\begin{definition}
Fix a concept class $C$ over domain $X$. A set $\{x_1, \dots, x_\ell\} \in X$ is \emph{shattered} by $C$ if for every labeling $b \in \{0, 1\}^\ell$, there exists $c \in C$ such that $b_1 = c(x_1), \dots, b_\ell = c(x_\ell)$. The \emph{Vapnik-Chervonenkis (VC) dimension} of $C$, denoted $\VC(C)$, is the size of the largest set which is shattered by $C$.
\end{definition}

The Vapnik-Chervonenkis (VC) dimension is an important combinatorial measure of a concept class.
Classical results in statistical learning theory show that the generalization error of a hypothesis $h$ and its empirical error (observed on a large enough sample) are similar.

\begin{definition}[Empirical Error] 
Let $S = ((x_i, \sigma_i))_{i = 1}^n \in (X\times\{0,1\})^n$ be a labeled sample. 
The {\em empirical error} of a hypothesis $h:X\rightarrow\{0,1\}$ w.r.t.\ $S$ is defined as 
$\error_{S}(h)= \frac{1}{n} |\{i : h(x_i)\neq \sigma_i\}|$.
\end{definition}

\begin{theorem}[VC-Dimension Generalization Bound, e.g.\ \cite{BlumerEhHaWa89}]\label{thm:Generalization}
Let $\DDD$ and $C$ be, respectively, a distribution and a concept class over a domain $X$, and let $c \in C$. For a sample $S=((x_i,c(x_i)))_{i=1}^n$ where $n\geq\frac{64}{\alpha}(\VC(C)\ln(\frac{64}{\alpha})+\ln(\frac{8}{\beta}))$
and the $x_i$ are drawn i.i.d. from $\DDD$, it holds that
\[\Pr\Big[\exists h\in C \text{ s.t.\ } \error_\DDD(h,c)> \alpha \ \land \ \error_S(h)\leq\frac{\alpha}{2} \Big] \le \beta.\]
\end{theorem}

\subsection{Conjunctions and Disjunctions}

\begin{definition}
Let $H$ be a concept class over a domain $X$, and let $k\in\N$. We use $H^{\vee k}$ to denote the class of all disjunctions (i.e., {\rm OR}) of at most $k$ concepts from $H$, and similarly, we denote $H^{\wedge k}$ for the class of all conjunctions (i.e., {\rm AND}) of at most $k$ concepts from $H$.
\end{definition}

The following observation is standard (see, e.g.,~\cite{EisenstatA07}).
\begin{observation}\label{obs:VCofOR}
For every concept class $H$ we have
$$\VC(H^{\vee k})\leq O(k\log(k)\cdot \VC(H)) \qquad\text{and}\qquad \VC(H^{\wedge k})\leq O(k\log(k)\cdot \VC(H)).$$
\end{observation}

Our strategy in the following sections for privately learning a concept class $C$ will be to use a ``simpler'' concept class $H$ such that for some (hopefully small) $k$ we have $C\subseteq H^{\vee k}$ or $C\subseteq H^{\wedge k}$.

\begin{example}
Let $\disj_{k,d}$ denote the class of all disjunctions (i.e., {\rm OR}) of at most $k$ literals over $d$ Boolean variables, and similarly, let $\conj_{k,d}$ denote the class of all conjunctions (i.e., {\rm AND}) of at most $k$ literals over $d$ Boolean variables. Trivially, $\disj_{k,d}=(\disj_{1,d})^{\vee k}$, and $\conj_{k,d}=(\conj_{1,d})^{\wedge k}$.
\end{example}

\subsection{Differential privacy}

Two databases $S,S'$ are called {\em neighboring} if they differ on a single entry.

\begin{definition}[Differential Privacy~\cite{DMNS06}] 
Let $\AAA$ be a randomized algorithm whose input is a database. Let $\eps,\delta \geq 0$. Algorithm $\AAA$ is $(\eps,\delta)$-differentially private if for all neighboring databases $S,S'$ and for any event $T$,
$$\Pr[\AAA(S) \in T] \leq e^\eps \cdot \Pr[\AAA(S') \in T] + \delta,$$ 
where the probability is taken over the coin tosses of the algorithm $\AAA$. When $\delta=0$ we omit it and say that $\AAA$ is $\eps$-differentially private.
\end{definition}

Our learning algorithms are designed via repeated applications of differentially private algorithms on a database. Composition theorems for differential privacy show that the price of privacy for multiple (adaptively chosen) interactions degrades gracefully.

\begin{theorem}[Composition of Differential Privacy \cite{DKMMN06, DL09, DRV10}]\label{thm:composition}
Let $0 < \eps, \delta' < 1$ and $\delta \in [0, 1]$. Suppose an algorithm $\AAA$ accesses its input database $S$ only through $m$ adaptively chosen executions of $(\eps, \delta)$-differentially private algorithms. Then $\AAA$ is 
\begin{enumerate}
\item $(m \eps, m\delta)$-differentially private, and
\item $(\eps', m \delta + \delta')$-differentially private for $\eps = \sqrt{2m\ln(1/\delta')} \cdot \eps + 2m\eps^2$.
\end{enumerate}
\end{theorem}

The most basic constructions of differentially private algorithms are via the Laplace mechanism as follows.

\begin{definition}[The Laplace Distribution]
A random variable has probability distribution $\Lap(b)$ if its probability density function is $f(x)=\frac{1}{2b}\exp(-\frac{|x|}{b})$, where $x\in\R$.
\end{definition}

\begin{definition}[Sensitivity]
A $f$ mapping databases to the reals has {\em sensitivity $s$} if for every neighboring $S,S$, it holds that $|f(S)-f(S')|\leq s$.
\end{definition}

\begin{theorem}[The Laplacian Mechanism \cite{DMNS06}]\label{thm:lap}
Let $f$ be a sensitivity $s$ function. The mechanism $\AAA$ that on input a database $S$ 
adds noise with distribution $\Lap(\frac{s}{\eps})$ to the output of $f(S)$ preserves $\eps$-differential privacy. Moreover,
$$\Pr\Big[|\AAA(S)-f(S)|>\Delta\Big]\leq \exp\left(-\frac{\epsilon \Delta}{s}\right).$$
\end{theorem}

We next describe the exponential mechanism of McSherry and Talwar~\cite{MT07}. Given a database $S$, the exponential mechanism privately chooses a ``good'' solution $h$ out of a set of possible solutions $H$ (in our context, $H$ will be a hypothesis class). This ``goodness'' is quantified using a \emph{quality function} that matches solutions to scores.

\begin{definition}[Quality function]
A \emph{quality function} $q=q(S,h)$ maps a database $S$ and a solution $h\in H$ to a real number, identified as the score of the solution $h$ w.r.t.\ the database $S$. We say that $q$ has sensitivity $s$ if $q(\cdot,h)$ has sensitivity $s$ for every $h\in H$.
\end{definition}

Given a sensitivity-1 quality function $q$ and a database $S$, the exponential mechanism chooses a solution $h\in H$ with probability proportional to $\exp\left(\epsilon \cdot q(S,h) /2 \right)$.

\begin{proposition}[Properties of the exponential mechanism]\label{prop:expMech}
(i) The exponential mechanism is $\eps$-differentially private. (ii)
Let $\lambda>0$. The exponential mechanism outputs a solution $h$ such that $q(S,h)\leq\max_{f\in H}\{q(S,f)\} - \lambda$ with probability at most $|H| \cdot \exp(-\eps \lambda /2)$.
\end{proposition}

\section{A Generic Construction via Set Cover}\label{sec:GenericConstruction}

In this section we present our generic construction for privately learning a concept class $C$ containing concepts that can be written as the conjunction or the disjunction of functions in a (hopefully simpler) class $H$. For readability we focus on conjunctions. The extension to disjunction is straightforward.

\begin{algorithm*}[t]

\caption{\bf\texttt{SetCoverLearner}}\label{alg:SetCoverLearner}

{\bf Settings: }Concept classes $C,H$ and an integer $k\in\N$ such that $C\subseteq H^{\wedge k}$.

\smallskip

\noindent {\bf Input:} Labeled sample $S=\{(x_i,\sigma_i)\}_{i=1}^n\in(X\times\{0,1\})^n$, privacy parameter $\eps$.

\smallskip

\noindent {\bf Tool used: }A selection procedure $\AAA$ that takes a database $S$ and a quality function $q$ (that assigns\\
\noindent \hphantom{\bf Tool used: }a score to each hypothesis in $H$), and returns a hypothesis $h\in H$.
\begin{enumerate}[leftmargin=15pt,rightmargin=10pt,itemsep=1pt]
\item For $j=1$ to $2k\log\frac{2}{\alpha}$
\begin{enumerate}
	\item Let $S^1$ and $S^0$ denote the set of positive and negative examples in $S$, respectively.
	\item For $h\in H$ let $\#_{h\rightarrow0}(S^1)$ and $\#_{h\rightarrow0}(S^0)$ denote the number of positive and negative examples in $S$, respectively, that $h$ labels as 0. That is, 
	$$\#_{h\rightarrow0}(S^1)=|\{x_i\in S^1: h(x_i)=0\}| \qquad\text{and}\qquad \#_{h\rightarrow0}(S^0)=|\{x_i\in S^0: h(x_i)=0\}|.$$
	\item\label{step:noisyCount} Let $w_j\leftarrow\left\lfloor\Lap\left(\frac{2k}{\eps}\log \frac{2}{\alpha}\right)\right\rfloor$ and set	$b_j= |S^0|+w_j-\frac{2k}{\eps}\log\left(\frac{2}{\alpha}\right)\log\left(\frac{2k}{\beta}\log \frac{2}{\alpha}\right)$.
	\item For every $h\in H$, define $q(h)=\min\left\{\#_{h\rightarrow0}(S^0) - \frac{b_j}{k}\;\;,\;\; -\#_{h\rightarrow0}(S^1)\right\}$.
	\item\label{step:selectionProcedure} Let $h_j\leftarrow\AAA(S,q)$, and delete from $S$ every $(x_i,\sigma_i)$ such that $h_j(x_i)=0$.
\end{enumerate}
\item Return the hypothesis $h_{fin}=h_1\wedge h_2\wedge\dots\wedge h_{2k\log \frac{2}{\alpha}}$.
\end{enumerate}
\end{algorithm*}

\begin{claim}\label{claim:smallEmpiricalError}
Fix a target function $c^*\in C$, and consider the execution of \texttt{SetCoverLearner} on a sample $S=\{(x_i,c^*(x_i))\}_{i=1}^n$. 
Assume that every run of the selection procedure $\AAA$ in Step~\ref{step:selectionProcedure} returns a hypothesis $h_j$ s.t.\ $q(h_j)\geq\max_{f\in H}\{q(f)\}-\lambda$. Then, with probability at least $1-\beta$ it holds that $h_{fin}$ errs on at most $\max\left\{\frac{\alpha n}{2},\; \frac{8k}{\eps}\log\left(\frac{2}{\alpha}\right)\log\left(\frac{2k}{\beta}\log \frac{2}{\alpha}\right)\right\}+4k\lambda\log\frac{2}{\alpha}$ example in $S$.
\end{claim}

\begin{proof}
First observe that there are $2k\log \frac{2}{\alpha}$ draws from $\Lap\left(\frac{2k}{\eps}\log \frac{2}{\alpha}\right)$ throughout the execution. By the properties of the Laplace distribution, with probability at least $1-\beta$ it holds that the maximum absolute value of these random variables is at most $\Delta=\frac{2k}{\eps}\log\left(\frac{2}{\alpha}\right)\log\left(\frac{2k}{\beta}\log \frac{2}{\alpha}\right)$. We continue with the analysis assuming that this is the case. In particular, this means that in every iteration $j$ we have $|S^0|-2\Delta\leq b_j\leq|S^0|$. Thus, in every iteration  there exists a hypothesis $\tilde{h}\in H$ with $q(\tilde{h})\geq0$. To see this, recall that the target concept $c^*$ can be written as $c^*=h^*_1\wedge\dots\wedge h^*_k$ for $h^*_1,\dots,h^*_k\in H$. Hence, in every iteration $j$ there is a hypothesis $\tilde{h}\in H$ that correctly classifies {\em all} of the (remaining) positive points in $S$ while correctly classifying at least $1/k$ fraction of the (remaining) negative points in $S$, i.e., at least $|S^0|/k\geq b_j/k$ negative points. Such a hypothesis $\tilde{h}$ satisfies $q(\tilde{h})=0$. By our assumption on the selection procedure $\AAA$, we therefore have that in each iteration $j$, the selection procedure identifies a hypothesis $h_j$ s.t.\ $q(h_j)\geq-\lambda$.

By the definition of $q$, in every iteration $j$ we have that the selected $h_j$ misclassifies at most $\lambda$ of the remaining positive examples in $S$. Therefore, $h_{fin}$ misclassifies at most $2k\lambda\log\frac{2}{\alpha}$ positive examples in $S$. Moreover, in every iteration $j$ s.t.\ $|S^0|\geq2k\lambda+4\Delta$ we have that $h_j$ classifies correctly at least $\frac{1}{2k}$ fraction of the negative examples in $S$. To see this, observe that as $q(h_j)\geq-\lambda$ we have
$$
\#_{h_j\rightarrow0}(S^0) \geq \frac{b_j}{k} -\lambda\geq\frac{|S^0|-2\Delta}{k}-\lambda\geq\frac{|S^0|}{2k}.
$$
That is, either there exists an iteration $j$ in which number of negative points in $S$ drops below 
$2k\lambda+4\Delta$, or every iteration shrinks the number of negative examples by 
a factor of $\frac{1}{2k}$, in which case after $2k\log\frac{2}{\alpha}$ iterations there could be at most $\frac{\alpha n}{2}$ negative 
points in $S$. Observe that $h_{fin}$ does not err on negative points that were removed from $S$, and therefore, there could be at most 
$\max\left\{\frac{\alpha n}{2},\; 2k\lambda+4\Delta\right\}$ negative points on which $h_{fin}$ errs. Overall, $h_{fin}$ errs on at most $\max\left\{\frac{\alpha n}{2},\; 4\Delta\right\}+4k\lambda\log\frac{2}{\alpha}$ points in $S$.
\end{proof}

Claim~\ref{claim:smallEmpiricalError} ensures that if at every step $\AAA$ picks $h_j$ of high quality, then (w.h.p.)\ algorithm \texttt{SetCoverLearner} returns a hypothesis from $H^{\wedge 2k\log\frac{2}{\alpha}}$ with low empirical error. Combining this with standard generalization bounds and with Observation~\ref{obs:VCofOR} (that bounds the VC dimension of $H^{\wedge 2k\log\frac{2}{\alpha}}$) we get the following theorem.

\begin{theorem}\label{thm:C_H_k}
Let $C,H,k$ be two concept classes and an integer such that $C\subseteq H^{\wedge k}$. Let $\AAA$ be a selection procedure that takes a database $S$ and a quality function $q$, and returns a hypothesis $h\in H$ such that $q(h_j)\geq\max_{f\in H}\{q(f)\}-\lambda$ with probability at least $1-\frac{\beta}{4k\log(2/\alpha)}$. Then, algorithm \texttt{SetCoverLearner} with $\AAA$ as the selection procedure is an $(\alpha,\beta)$-PAC learner for $C$ with sample complexity
$$n=\Theta\left(\frac{k\log\frac{1}{\alpha}}{\alpha}\left(\VC(H)\log(k)+\lambda+\frac{1}{\eps}\log\left(\frac{k}{\beta}\log\frac{1}{\alpha}\right)\right)\right).$$
\end{theorem}

\subsection{Tuning the selection procedure}

If $H$ is finite, then one could directly implement the selection procedure $\AAA$ using the exponential mechanism 
 of McSherry and Talwar~\cite{MT07} to find a hypothesis $h_j$ with large $q(h_j)$ at each iteration. In order to guarantee that all of the $\approx k$ iterations of algorithm \texttt{SetCoverLearner} satisfy together $(\eps,\delta)$-differential privacy, it suffices that each application of the exponential mechanism satisfies $\hat{\eps}\approx\frac{\eps}{\sqrt{k}}$-differential privacy (see Theorem~\ref{thm:composition}). When choosing such an $\hat{\eps}$, the exponential mechanism identifies, in every iteration, an $h_j$ such that 
$$q(j)\gtrsim\max_{f\in H}\{q(f)\}-\frac{1}{\hat{\eps}}\log|H|\approx\max_{f\in H}\{q(f)\}-\frac{\sqrt{k}}{\eps}\log|H|.$$
This gives a selection procedure $\AAA$ which selects $h_j$ with $q(h_j)\geq\max_{f\in H}\{q(f)\}-\lambda$, for $\lambda\approx\frac{\sqrt{k}}{\eps}\log|H|$.

\begin{example}
There exist efficient $(\eps,\delta)$-differentially private $(\alpha,\beta)$-PAC learners for $\conj_{k,d}$ and for $\disj_{k,d}$ with sample complexity $n=\tilde{\Theta}\left( \frac{1}{\alpha\eps} \cdot k^{1.5}\log d \right)$.
\end{example}

The reason for the dependency in $k^{1.5}$ in the above example, is that for the privacy analysis we wanted to ensure that each iteration was differentially private with parameter $\approx\eps/\sqrt{k}$ (because when composing $\ell$ differentially private mechanisms the privacy budget deteriorates proportionally to $\sqrt{\ell}$). This resulted in $\approx\sqrt{k}/\eps$ misclassified points per iteration (and there are $\approx k$ iterations). As we next explain, in our case, the privacy parameter does not need to deteriorate with the number of iterations, which allows us to improve the sample complexity by a $\sqrt{k}$ factor. Our approach to proving this improved bound builds on the analysis of Gupta et al.~\cite{GuptaLMRT10} for their private algorithm for set cover. The main difference is that we have both positive and negative examples, which we need to handle differently. As we next explain, this will be achieved using Step~\ref{step:noisyCount} of \texttt{SetCoverLearner}.

\begin{claim}\label{claim:SetCoverLearnerPrivacy}
Let $\eps\in(0,1)$ and $\delta<1/e$. Instantiating \texttt{SetCoverLearner} with the exponential mechanism as the selection procedure $\AAA$ with privacy parameter $\hat{\eps}=\frac{\eps}{2\ln(e/\delta)}$ (for each iteration) satisfies $(\eps,\delta)$-differential privacy.
\end{claim}

We first present an intuitive (and oversimplified) overview of the proof. 
Consider two neighboring databases $S$ and $S'$ such that $S'=S\cup\{(x^*,\sigma^*)\}$, and let us focus this intuitive overview on the case where $\sigma^*=0$. Fix a possible output $\vec{h}=(h_1,h_2,\dots,h_{2k\log\frac{2}{\alpha}})$ of \texttt{SetCoverLearner}. We will analyze the ratio
\begin{eqnarray}
\frac{\Pr[\texttt{SetCoverLearner}(S)=\vec{h}]}{\Pr[\texttt{SetCoverLearner}(S')=\vec{h}]}.
\label{eq:IntuitiveRatio}
\end{eqnarray}
Let $t$ be such that $h_t$ is the first hypothesis in this output vector satisfying $h_t(x^*)=0$. Observe that after the $t$th iteration, the executions on $S$ and on $S'$ continue exactly the same, since $(x^*,\sigma^*)$ is removed from $S'$ during the $t$th iteration (because in every iteration we remove all input elements on which the selected hypothesis evaluates to 0). Intuitively, if $t$ is small then we only need to pay (in the privacy analysis) for a small number of iterations. In general, however, $t$ might be as large as $2k\log\frac{2}{\alpha}$, and accounting for that many iterations in the privacy analysis is exactly what we are trying to avoid.

Recall that each iteration $j$ of \texttt{SetCoverLearner} draws a random noise $w_j$ from $\left\lfloor\Lap\left(\frac{2k}{\eps}\log \frac{2}{\alpha}\right)\right\rfloor$. Let us denote these noises as they are in the execution on $S$ as $\vec{w}=(w_1,\dots,w_{2k\log\frac{2}{\alpha}})$ and in the execution on $S'$ as $\vec{w}'=(w'_1,\dots,w'_{2k\log\frac{2}{\alpha}})$. Furthermore, let us assume that $w'_j=w_j-1$ for every $j\leq t$ and that $w'_j=w_j$ for every $j>t$. By the properties of the Laplace distribution, this assumption distorts our bound on the ratio in expression~(\ref{eq:IntuitiveRatio}) by at most an $e^\eps$ factor. (In a sense, for these random noises we {\em do} account for all $2k\log\frac{2}{\alpha}$ potential iterations by sampling random noises with larger variance. However, this larger variance is mitigated by the fact that in the quality function $q$ we divide noises by $k$, and hence, we {\em do not} incurr an increase of $\poly(k)$ in the sample complexity due to this issue.)

We have already established that after the $t$th iteration, the two executions are identical. In addition, due to our assumption on $\vec{w}$ and $\vec{w}'$, during the first $t$ iterations, the only hypotheses with different qualities (between the execution on $S$ and on $S'$) or those hypotheses that label $x^*$ as 0. This is because if a hypothesis $h$ labels $x^*$ as 1, then $(x^*,0)$ only effects the quality $q(h)$ via the noisy estimation for the size of $S^0$ (denoted as $b_j$ in the algorithm), which by our assumption on $\vec{w}$ and $\vec{w}'$ is the same in the two executions (because the difference in the noise cancels out the difference in $|S^0|$). To summarize, after conditioning on $\vec{w}$ and $\vec{w}'$, the additional example $(x^*,0)$ causes the two executions to differ only in their first $t$ iterations, and within these $t$ iterations it affects only the qualities of the hypotheses that label $x^*$ as zero. This can be formalized to bound the ratio in expression~(\ref{eq:IntuitiveRatio}) by $\lesssim \prod_{j=1}^{t}\exp\left( \eps\cdot p_j \right)$, where $p_j$ is the probability that a hypothesis that labels $x^*$ as 0 is chosen at step $j$ of the algorithm. The proof then concludes by arguing that if these probabilities $\{p_j\}$ are small then they reduce our privacy costs (since they multiply $\eps$), and if these probabilities $\{p_j\}$ are large then the index $t$ should be small (since we are likely to identify a hypothesis that labels $x^*$ as zero quickly, and $t$ is the index of the first such hypothesis), and therefore we must only account for the privacy loss incurred during a small number of iterations. We now proceed with the formal proof.

\begin{proof}[Proof of Claim~\ref{claim:SetCoverLearnerPrivacy}]
Let $S$ and $S'$ be two neighboring databases such that $S\triangle S'=\{(x^*,\sigma^*)\}$. 
Fix a possible output of \texttt{SetCoverLearner} $\vec{h}=(h_1,h_2,\dots,h_{2k\log\frac{2}{\alpha}})$, and let $q_{j,S,w_j}(h)$
denote the quality $q(h)$ of a hypothesis $h\in H$ during the $j$th iteration of the algorithm when running on $S$, conditioned on $h_1,\dots,h_{j-1}$ being chosen in the previous steps and on the value of $w_j$. 
Let $t$ be such that $h_t$ is the first hypothesis in $\vec{h}$ satisfying $h_t(x^*)=0$.

\paragraph{Case (a):} $S'=S\cup\{(x^*,\sigma^*)\}$ and $\sigma^*=1$. 
Fix a noise vector $\vec{w}$. We can calculate
\begin{eqnarray*}
\frac{\Pr[\texttt{SetCoverLearner}(S)=\vec{h}|\vec{w}]}{\Pr[\texttt{SetCoverLearner}(S')=\vec{h}|\vec{w}]} 
&=& \prod_{j=1}^{2k\log\frac{2}{\alpha}}\left( \frac{\exp(\hat{\eps}\cdot q_{j,S,w_j}(h_j))/\left(\sum_{f\in H}\exp(\hat{\eps}\cdot q_{j,S,w_j}(f))\right)}{\exp(\hat{\eps}\cdot q_{j,S',w_j}(h_j))/\left(\sum_{f\in H}\exp(\hat{\eps}\cdot q_{j,S',w_j}(f))\right)} \right)\\
&=& \frac{\exp(\hat{\eps}\cdot q_{t,S,w_t}(h_t))}{\exp(\hat{\eps}\cdot q_{t,S',w_t}(h_t))}\cdot\prod_{j=1}^{t}\left(\frac{\sum_{f\in H}\exp(\hat{\eps}\cdot q_{j,S',w_j}(f))}{\sum_{f\in H}\exp(\hat{\eps}\cdot q_{j,S,w_j}(f))}\right)
\end{eqnarray*}
After $t$, the remaining elements in $S$ and $S'$ are identical, and all subsequent terms cancel. Moreover, except for the $t$th term, the numerators of both the top and the bottom expressions cancel, since all the relevant scores are equal.

We are currently assuming that $S'=S\cup\{(x^*,\sigma^*)\}$ and $\sigma^*=1$. Hence, for every hypothesis $f$ and every step $j$ we have that  $q_{j,S,w_j}(f)-1\leq q_{j,S',w_j}(f) \leq q_{j,S,w_j}(f)$, since adding a positive example to the database can decrease the quality by at most 1 (and it cannot increase the quality). Hence, the first term above is at most $\exp(\hat{\eps})$, and the second term is at most 1. 
As this holds for every possible value of the noise vector $\vec{w}$, we get that
$$
\frac{\Pr[\texttt{SetCoverLearner}(S)=\vec{h}]}{\Pr[\texttt{SetCoverLearner}(S')=\vec{h}]} \leq\exp(\hat\eps).
$$

\paragraph{Case (b):} $S'=S\cup\{(x^*,\sigma^*)\}$ and $\sigma^*=0$.
Fix a noise vector $\vec{w}$, and let $\vec{w}'$ be such that $w'_j=w_j-1$ for every $j\leq t$ and $w'_j=w_j$ for every $j> t$. (Recall that $t$ is the index of the first hypothesis in the output vector $\vec{h}$ that labels $x^*$ as 0.) We have that
\begin{eqnarray*}
\frac{\Pr[\texttt{SetCoverLearner}(S)=\vec{h}|\vec{w}]}{\Pr[\texttt{SetCoverLearner}(S')=\vec{h}|\vec{w}']} 
&=& \prod_{j=1}^{2k\log\frac{2}{\alpha}}\left( \frac{\exp(\hat{\eps}\cdot q_{j,S,w_j}(h_j))/\left(\sum_{f\in H}\exp(\hat{\eps}\cdot q_{j,S,w_j}(f))\right)}{\exp(\hat{\eps}\cdot q_{j,S',w'_j}(h_j))/\left(\sum_{f\in H}\exp(\hat{\eps}\cdot q_{j,S',w'_j}(f))\right)} \right)\\
&=& \frac{\exp(\hat{\eps}\cdot q_{t,S,w_t}(h_t))}{\exp(\hat{\eps}\cdot q_{t,S',w'_t}(h_t))}\cdot\prod_{j=1}^{t}\left(\frac{\sum_{f\in H}\exp(\hat{\eps}\cdot q_{j,S',w'_j}(f))}{\sum_{f\in H}\exp(\hat{\eps}\cdot q_{j,S,w_j}(f))}\right)
\end{eqnarray*}
As before, after $t$ the remaining elements in $S$ and $S'$ are identical, and all subsequent terms cancel (recall that $w_j=w'_j$ for every $j>t$). 
In addition, due to our choice of $w'_j=w_j-1$ for every $j\leq t$, we again get that, except for the $t$th term, the numerators of both the top and the bottom expression cancel, since all the relevant scores are equal.

We are currently analyzing the case where $S'=S\cup\{(x^*,\sigma^*)\}$ and $\sigma^*=0$. Hence, the first term above is $\exp(-\hat{\eps})<1$, because $q_{t,S',w'_t}(h_t)=q_{t,S,w_t}(h_t)+1$. Moreover, for every $j\leq t$ we have that $w'_j=w_j-1$. Hence, for every $j\leq t$ and every hypothesis $f$ s.t.\ $f(x^*)=1$ we have $q_{j,S',w'_t}(f)=q_{j,S,w_t}(f)$. Also, for every $j\leq t$ and every hypothesis $f$ s.t.\ $f(x^*)=0$ we have $q_{j,S',w'_t}(f)=q_{j,S,w_t}(f)+1$. Therefore we have
\begin{eqnarray}
&&\frac{\Pr[\texttt{SetCoverLearner}(S)=\vec{h}|\vec{w}]}{\Pr[\texttt{SetCoverLearner}(S')=\vec{h}|\vec{w}']}\nonumber\\[1em] 
&&\leq \prod_{j=1}^{t}\left( \frac{(\exp(\hat\eps)-1)\cdot\sum\limits_{\substack{f\in H:\\f(x^*)=0}}\exp(\hat{\eps}\cdot q_{j,S,w_j}(f))+\sum\limits_{f\in H}\exp(\hat{\eps}\cdot q_{j,S,w_j}(f))}{\sum\limits_{f\in H}\exp(\hat{\eps}\cdot q_{j,S,w_j}(f))} \right)\nonumber\\[1em]
&&=\prod_{j=1}^{t}\left( 1+ (\exp(\hat\eps)-1)\cdot p_j(S,w_j) \right)\label{eq:1}
\end{eqnarray}
where $p_j(S,w_j)$ is the probability that a hypothesis that labels $x^*$ as 0 is chosen at step $j$ of the algorithm running on $S$, conditioned on picking the hypotheses $h_1,\dots,h_{j-1}$ in the previous steps, and on the noise $w_j$.

For an instance $S$ and an example $x^*$, we say that an output $\vec{h}=(h_1,\dots,h_{2k\log\frac{2}{\alpha}})$ is {\em $\lambda$-bad} if $\sum_{j=1}^{2k\log\frac{2}{\alpha}} p_j(S,w_j)\cdot \1{\{h_1(x^*)=h_2(x^*)=\dots=h_j(x^*)=1\}}>\lambda$, where $p_j(S,w_j)$ is as defined above. We call the output $\vec{h}$ {\em $\lambda$-good} otherwise. We first consider the case when the output $\vec{h}$ is $\ln(1/\delta)$-good. By the definition of $t$ we have
$$
\sum_{j=1}^{t-1} p_j(S,w_j)\leq\ln(1/\delta).
$$
Then we can bound the expression (\ref{eq:1}) by
\begin{eqnarray*}
\frac{\Pr[\texttt{SetCoverLearner}(S)=\vec{h}|\vec{w}]}{\Pr[\texttt{SetCoverLearner}(S')=\vec{h}|\vec{w}']} 
&\leq& \prod_{j=1}^{t}\left( 1+ (\exp(\hat\eps)-1)\cdot p_j(S,w_j) \right)\\
&\leq& \exp\left( 2\hat\eps\cdot\sum_{j=1}^t p_j(S,w_j) \right)\\
&\leq& \exp\left( 2\hat\eps\cdot\left(\ln(1/\delta)+p_t(S,w_t)\right) \right)\\
&\leq& \exp\left( 2\hat\eps\cdot\left(\ln(1/\delta)+1\right) \right)\\
&\leq&\exp(\eps).
\end{eqnarray*}
So, for every $\ln(1/\delta)$-good output $\vec{h}$ for $S$ we have
\begin{align*}
&\Pr[\texttt{SetCoverLearner}(S)=\vec{h}]\\
&=\sum_{\substack{w_1,\dots,w_t,\\w_{t+1},\dots,w_{2k\log\frac{2}{\alpha}}}}\Pr\left[
\begin{array}{l}
	w_1,\dots,w_t,\\
	w_{t+1},\dots,w_{2k\log\frac{2}{\alpha}}
\end{array}
\right]\cdot \Pr\left[\texttt{SetCoverLearner}(S)=\vec{h}\left|\begin{array}{l}
	w_1,\dots,w_t,\\
	w_{t+1},\dots,w_{2k\log\frac{2}{\alpha}}
\end{array}\right.\right]\\
&=\sum_{\substack{w_1,\dots,w_t,\\w_{t+1},\dots,w_{2k\log\frac{2}{\alpha}}}}\Pr\left[
\begin{array}{l}
	w_1+1,\dots,w_t+1,\\
	w_{t+1},\dots,w_{2k\log\frac{2}{\alpha}}
\end{array}
\right]\cdot \Pr\left[\texttt{SetCoverLearner}(S)=\vec{h}\left|\begin{array}{l}
	w_1+1,\dots,w_t+1,\\
	w_{t+1},\dots,w_{2k\log\frac{2}{\alpha}}
\end{array}\right.\right]\\
&\leq\sum_{\substack{w_1,\dots,w_t,\\w_{t+1},\dots,w_{2k\log\frac{2}{\alpha}}}}e^{\eps}\cdot\Pr\left[
\begin{array}{l}
	w_1,\dots,w_t,\\
	w_{t+1},\dots,w_{2k\log\frac{2}{\alpha}}
\end{array}
\right]\cdot \Pr\left[\texttt{SetCoverLearner}(S)=\vec{h}\left|\begin{array}{l}
	w_1+1,\dots,w_t+1,\\
	w_{t+1},\dots,w_{2k\log\frac{2}{\alpha}}
\end{array}\right.\right]\\
&\leq\sum_{\substack{w_1,\dots,w_t,\\w_{t+1},\dots,w_{2k\log\frac{2}{\alpha}}}}e^{\eps}\cdot\Pr\left[
\begin{array}{l}
	w_1,\dots,w_t,\\
	w_{t+1},\dots,w_{2k\log\frac{2}{\alpha}}
\end{array}
\right]\cdot e^{\eps}\cdot\Pr\left[\texttt{SetCoverLearner}(S')=\vec{h}\left|\begin{array}{l}
	w_1,\dots,w_t,\\
	w_{t+1},\dots,w_{2k\log\frac{2}{\alpha}}
\end{array}\right.\right]\\
&=e^{2\eps}\cdot \Pr[\texttt{SetCoverLearner}(S')=\vec{h}]
\end{align*}

As for a $\ln(1/\delta)$-bad output, the following lemma shows the probability that $\texttt{SetCoverLearner}(S)$ outputs a $\ln(1/\delta)$-bad output (for $S$) is at most $\delta$.

\begin{lemma}[\cite{GuptaLMRT10}]\label{lemma:GuptaRandom}
Consider the following $n$ round probabilistic process. In each round, an adversary chooses a $p_j\in [0, 1]$ possibly based on the first $(j-1)$ rounds and a coin is tossed with heads probability $p_j$. Let $Z_j$ be the indicator for the the event that no coin comes up heads in the first $j$ steps. Let $Y$ denote the random variable $\sum_{j=1}^n p_j Z_j$.
Then for any $y$ we have $\Pr[Y>y]\leq\exp(-y)$.
\end{lemma}

Specifically, to map our setting to that of Lemma~\ref{lemma:GuptaRandom}, consider running \texttt{SetCoverLearner} as follows. When choosing a hypothesis $h_j$ in step $j$, the algorithm first tosses a
coin whose heads probability is $p_j(S,w_j)$ to decide whether to pick a hypothesis that labels $x^*$ as 0 or not. Then it uses a
second source of randomness to determine the hypothesis $h_j$ itself, sampling with the appropriate conditional probabilities based on the outcome of the coin. 

Thus, for any set $F$ of outcomes, we have
\begin{eqnarray*}
&&\Pr[\texttt{SetCoverLearner}(S)\in F] = \sum_{\vec{f}\in F}\Pr\left[\texttt{SetCoverLearner}(S)=\vec{f}\right]\\
&&= \sum_{\substack{\vec{f}\in F: \vec{f} \text{ is}\\ \ln(1/\delta)\text{-bad}\\ \text{for } S}}\Pr\left[\texttt{SetCoverLearner}(S)=\vec{f}\right] + \sum_{\substack{\vec{f}\in F: \vec{f} \text{ is}\\ \ln(1/\delta)\text{-good}\\ \text{for } S}}\Pr\left[\texttt{SetCoverLearner}(S)=\vec{f}\right]\\
&&\leq \delta + \sum_{\substack{\vec{f}\in F: \vec{f} \text{ is}\\ \ln(1/\delta)\text{-good } \text{for } S}}e^{2\eps}\cdot\Pr\left[\texttt{SetCoverLearner}(S')=\vec{f}\right]\\
&&\leq e^{2\eps}\cdot\Pr[\texttt{SetCoverLearner}(S')\in F]+\delta.
\end{eqnarray*}

A similar analysis holds for the case where $S=S'\cup\{(x^*,\sigma^*)\}$.
\end{proof}

For example, by combining Claim~\ref{claim:SetCoverLearnerPrivacy} with Claim~\ref{claim:smallEmpiricalError}, we get improved learners for conjunctions and disjunctions:

\begin{theorem}
There exist efficient $(\eps,\delta)$-differentially private $(\alpha,\beta)$-PAC learners for $\conj_{k,d}$ and $\disj_{k,d}$ with sample complexity $n=\tilde{O}\left( \frac{1}{\alpha\eps} \cdot k \log d\right)$.
\end{theorem}

\section{Convex Polygons in the Plane}\label{sec:convexPolygons}

In this section we show how our generic construction from the previous section applies to {\em convex} polygons in a (discrete version of the) Euclidean plane. This is an important step towards our construction for (not necessarily convex) polygons.

We represent a convex polygon with $k$ edges as the intersection of $k$ halfplanes. 
A halfplane over $\R^2$ can be represented using 3 parameters $a,b,c\in\R$ with $f_{a,b,c}(x,y)=1$ iff $cy\geq ax + b$. Denote the set of all such halfplanes over $\R^2$ as 
$$\halfplane=\{f_{a,b,c} : a,b,c\in\R\},\quad\text{where } f_{a,b,c}(x,y)=1 \text{ iff } cy\geq ax + b.$$
We can now define the class of convex polygons with $k$ edges over $\R^2$ as
$$\ckgon=\halfplane^{\wedge k}.$$

For a parameter $d\in\N$, let $X_d=\{0,1,2,\dots,d\}$, and let $X_d^2=(X_d)^2$ denote a discretization of the Euclidean plane, in which each axis consists of the points in $X_d$. 
We assume that our examples are from $X_d^2$.
Hence, as explained next, we are able to represent a halfplane using only two real parameters $a,b\in\R$ and a bit $z\in\{\pm1\}$.
The parameters $a$ and $b$ define the line $y=ax+b$, and the parameter $z$ determines whether the halfplane is ``above'' or ``below'' the line. In other words, $f_{a,b,z}(x,y)=1$ iff $zy \geq z(ax + b)$. Even though in this representation we do not capture vertical lines, for our purposes, vertical lines will not be needed. The reason is that when the examples come from the discretization $X_d^2$, a vertical line can always be replaced with a non-vertical line such that the corresponding halfplanes behave exactly the same on all of $X_d^2$. Moreover, since the discretization $X^2_d$ is finite, it suffices to consider {\em bounded} real valued parameters $a,b\in[-2d^2,2d^2]$ (see Observation~\ref{obs:halfplane} below). Actually, by letting $a$ reside in a bigger range, we can encode the bit $z$ in $a$, and represent a halfplane using only two real numbers. 
We denote the set of all such halfplanes as
\begin{eqnarray*}
&&\halfplane_d=\left\{f_{\hat{a},b} : -2d^2\leq \hat{a}\leq6d^2,\; -2d^2\leq b\leq2d^2 \right\},\\
&&\text{where }f_{\hat{a},b}(x,y)=1 \text{ iff } zy \geq z(ax + b) \text{ for } a=\hat{a}-4d^2\cdot\1_{\{\hat{a}>2d^2\}} \text{ and } z=1-2\cdot\1_{\{\hat{a}>2d^2\}}.
\end{eqnarray*}

\begin{observation}\label{obs:noZ}
Let $a,b\in[-2d^2,2d^2]$ and $z\in\{\pm1\}$, and define $f_{a,b,z}(x,y)=1$ iff $zy\geq z(ax+b)$. Then, there exists an $\hat{f}\in\halfplane_d$ such that $\hat{f}\equiv f_{a,b,z}(x,y)$.
\end{observation}

\begin{proof}[Proof sketch]
If $z=1$ then define $\hat{a}=a$. Otherwise, if $z=-1$ then define $\hat{a}=a+4d^2$. Observe that in both cases $f_{\hat{a},b}\in\halfplane_d$ is equivalent to $f_{a,b,z}$.
\end{proof}

\begin{observation}\label{obs:halfplane}
For every $f\in\halfplane$ there exists an $\hat{f}\in\halfplane_d$ such that for every $(x,y)\in X^2_d$ we have $f(x,y)=\hat{f}(x,y)$.
\end{observation}

\begin{proof}[Proof sketch]
Let $f\in\halfplane$. 
By Observation~\ref{obs:noZ}, it suffices to show that there exists a halfplane $\hat{f}_{a,b,z}$ equivalent to $f$ of the form $\hat{f}_{a,b,z}(x,y)=1$ iff $zy\geq z(ax+b)$, where $a,b\in[-2d^2,2d^2]$. Without loss of generality, we may assume that $f$ ``touches'' two points $(x_1,y_1),(x_2,y_2)\in X^2_d$, as otherwise we could ``tilt'' $f$ to make it so without effecting the way it labels points in $X^2_d$. Hence, $f$ can be defined by the line equation $(y-y_1)(x_2-x_1)=(y_2-y_1)(x-x_1)$, together with a bit $z\in\{\pm1\}$ that determines whether the halfplane is ``above'' or ``below'' that line. First observe that if $x_1\neq x_2$, then this line equation can be rewritten as
$$y=\frac{y_2-y_1}{x_2-x_1}x+\left(y_1-x_1\frac{y_2-y_1}{x_2-x_1}\right)\triangleq ax+b,$$
where $a,b\in[-d^2,d^2]$ because $x_1,y_1,x_2,y_2\in X_d$ and $x_1\neq x_2$. That is, the halfplane $f$ can be defined as $f(x,y)=1$ iff $zy\geq z(ax+b)$, as required.
Next note that if $x_1=x_2$, then $f$ is described by the vertical line $x=x_1$ and a bit $z\in\{\pm1\}$, where $f(x,y)=1$ iff $zx\geq zx_1$. 
Now consider the line that passes through $(x_1,0)$ and $(x_1+0.5,d)$, and a line that passes through $(x_1,0)$ and $(x_1-0.5,d)$. One of these two lines, depending on $z$, defines a halfplane that splits $X_d^2$ identically to $f$. Such a line can be described as $zy=z(ax+b)$ for $a,b\in[-2d^2,2d^2]$.
\end{proof}

\begin{remark}
We think of the discretization size $d$ as a large number, e.g., $d=2^{64}$. The runtime and the sample complexity of our algorithms is at most logarithmic in $d$. 
\end{remark}

A consequence of Observation~\ref{obs:halfplane} is that, in order to learn $\ckgon$ over examples in $X^2_d$, it suffices to describe a learner for the class $\halfplane_d^{\wedge k}$ over examples in $X^2_d$. As we next explain, this can be done using our techniques from Section~\ref{sec:GenericConstruction}. Concretely, we need to specify the selection procedure used in Step~1e of algorithm \texttt{SetCoverLearner}, for privately choosing a hypothesis from $\halfplane_d$. Our selection procedure appears in algorithm \texttt{SelectHalfplane}.

\begin{algorithm*}[t]

\caption{\bf\texttt{SelectHalfplane}}\label{alg:SelectHalfplane}

\noindent {\bf Input:} Labeled sample $S=\{((x_i,y_i),\sigma_i)\}_{i=1}^n\in(X^2_d\times\{\pm1\})^n$, privacy parameter $\eps$, quality function $q:\halfplane_d\rightarrow\R$.

\begin{enumerate}[leftmargin=15pt,rightmargin=10pt,itemsep=1pt]

\item Denote $D=\left[-2d^2,2d^2\right]$ and $F=\left[-2d^2,6d^2\right]$. We will refer to the axes of $D^2$ and of $F\times D$ as $a$ and $b$.
\item Identify every example $((x,y),\sigma)\in S$ with the line $\ell_{x,y}$ in $D^2$ defined by the equation $y=xa+b$, where $a,b$ are the variables and $x,y$ are the coefficients. Denote $S_{\rm dual}=\{\ell_{x,y} : ((x,y),\sigma)\in S\}$.
\item Let $R=\{r^1_1,r^1_2,\dots,r^1_{|R|}\}$ denote the partition of $D^2$ into regions defined by the lines in $S_{\rm dual}$. Also let $R'=\{r^2_1,\dots,r^2_{|R|}\}$ be a partition of $[2d^2,6d^2]\times D$ identical to $R$ except that it is shifted by $4d^2$ on the $a$ axis. Denote $\hat{R}=R\cup R'$.\\
\gray{\% Note that, by induction, $n$ lines can divide the plane into at most $n^2$ different regions. Hence, $|R|$ is small.}
\item For every $1\leq i\leq |R|$, let $w_i$ denote the area of region $r^1_i$ (which is the same as the area of $r^2_i$), and let $(a^1_i,b^1_i)\in r^1_i$ and $(a^2_i,b^2_i)\in r^2_i$ be arbitrary points in these regions.
\item Denote $N=\sum_{r_i^j\in\hat{R}}w_i\cdot\exp(\eps\cdot q(f_{a^j_i,b^j_i}))$, where $f_{a^j_i,b^j_i}$ is a halfplane in $\halfplane_d$. 
\item Choose and return a pair $(\hat{a},b)\in [-2d^2,6d^2]\times[-2d^2,2d^2]$ with probability density function 
 $p(\hat{a},b)=\frac{1}{N}\cdot\exp(\eps\cdot q(f_{\hat{a},b}))$.\\
\gray{\% Note that for every $(a,b),(a',b')\in r^j_i$ in the same region we have $q(f_{a,b})=q(f_{a',b'})$ (see Observation~\ref{obs:sameRegion}). Hence, this step can be implemented by first selecting a region $r^j_i\in\hat{R}$ with probability proportional to $w_i\cdot\exp(\eps\cdot q(f_{a^j_i,b^j_i}))$, and then selecting a random $(a,b)\in r^j_i$ uniformly.}
\end{enumerate}
\end{algorithm*}

\paragraph{Privacy analysis of \texttt{SelectHalfplane}.}
Consider running algorithm \texttt{SelectHalfplane} with a score function $q$ whose sensitivity is (at most) 1, and observe that, as in the standard analysis of the exponential mechanism~\cite{MT07}, algorithm \texttt{SelectHalfplane} satisfies $2\eps$-differential privacy. To see this, fix two neighboring databases $S,S'$, and denote the probability density functions in the execution on $S$ and on $S'$ as $p_S(\hat{a},b)$ and $p_{S'}(\hat{a},b)$, respectively. Since $q$ is of sensitivity 1, for every $(\hat{a},b)\in [-2d^2,6d^2]\times[-2d^2,2d^2]$ we have that $p_{S}(\hat{a},b)\leq e^{2\eps} p_{S'}(\hat{a},b)$. Hence, for any set of possible outcomes $F$ we have
$\Pr[\texttt{SelectHalfplane}(S)\in F]\leq e^{2\eps}\cdot \Pr[\texttt{SelectHalfplane}(S')\in F]$, as required. Moreover, a similar analysis to that of Claim~\ref{claim:SetCoverLearnerPrivacy} shows the following.
\begin{claim}\label{claim:SelectHalfplanePrivacy}
When instantiating algorithm \texttt{SetCoverLearner} with \texttt{SelectHalfplane} as the selection procedure, in order for the whole execution to satisfy $(\eps,\delta)$-differential privacy, it suffices to execute each instance of \texttt{SelectHalfplane} with a privacy parameter $\hat{\eps}=O\left(\eps/\log(1/\delta)\right)$.
\end{claim}

\subsection{Utility analysis of \texttt{SelectHalfplane}}

In algorithm \texttt{SelectHalfplane} we identify points in $X^2_d$ with lines in $D^2$ and vice verse. The following observation states that if two points in $D^2$ belong to the same region (as defined in Step~3) then these two points correspond to halfplanes in $X^2_d$ that agree on every point in the input sample $S$. This allows us to partition the halfplanes (in the primal plane) into a small number of equivalence classes.

\begin{observation}\label{obs:sameRegion}
Consider the execution of \texttt{SelectHalfplane} on a sample $S$, and let $\hat{R}=\{r^j_i\}$ be the regions defined in Step~3 (for $j\in\{1,2\}$). For every region $r^j_i\in \hat{R}$, for every two points in this region $(a_1,b_1),(a_2,b_2)\in r^j_i$, and for every example $(x,y)$ in the sample $S$ we have $f_{a_1,b_1}(x,y)=f_{a_2,b_2}(x,y)$.
\end{observation}

\begin{proof}
Fix two points $(a_1,b_1),(a_2,b_2)$ that belong to the same region in $\hat{R}$. By the definition of the regions in $\hat{R}$, for every example $(x,y)$ in the sample $S$ we have that
$$y\geq a_1 x+b_1 \qquad \text{iff} \qquad y\geq a_2 x+b_2,$$
and hence, $f_{a_1,b_1}(x,y)=1$ iff $f_{a_2,b_2}(x,y)=1$.
\end{proof}

In particular, Observation~\ref{obs:sameRegion} shows that the function $p$ defined in Step~6 indeed defines a probability density function, as for $F=[-2d^2,6d^2]$ and $D=[-2d^2,2d^2]$ we have
\begin{eqnarray*}
\int_{F\times D}{p(a,b)}\;{\rm d}^2(a,b) &=& \sum_{r_i^j\in \hat{R}} \int_{r^j_i}{p(a,b)}\;{\rm d}^2(a,b)
= \sum_{r^j_i\in \hat{R}} \int_{r^j_i}{\frac{\exp(\eps\cdot q(f_{a,b}))}{N}}\;{\rm d}^2(a,b)\\
&=& \sum_{r^j_i\in \hat{R}} \int_{r^j_i}{\frac{\exp(\eps\cdot q(f_{a^j_i,b^j_i}))}{N}}\;{\rm d}^2(a,b)
= \sum_{r^j_i\in \hat{R}} \frac{\exp(\eps\cdot q(f_{a^j_i,b^j_i}))}{N}\int_{r^j_i}{1}\;{\rm d}^2(a,b)\\
&=& \sum_{r^j_i\in \hat{R}} \frac{w_i\cdot\exp(\eps\cdot q(f_{a^j_i,b^j_i}))}{N}=1.
\end{eqnarray*}

We also need to argue about the {\em area} of the region in the dual plane that corresponds to hypotheses with high quality (as the probability of a choosing a hypotheses from that region is proportional to its area). This is done in the following claim.

\begin{claim}
Consider the execution of \texttt{SelectHalfplane} on a sample $S$, and let $w_1,\dots,w_{|R|}$ denote the areas of the regions defined in Step~3. Then for every $i$ we have that $w_i\geq d^{-4}/4$.
\end{claim}

\begin{proof}
We will show that every two different vertices of the regions in $R$ are at distance at least $1/d^2$, and hence, the minimal possible area is that of a equilateral triangle with edge length $1/d^2$, which has area $\frac{\sqrt{3}}{4d^4}$.

To show this lower bound on the distance between a pair of vertices, let $\ell_{x_1,y_1},\ell_{x_2,y_2},\ell_{x_3,y_3},\ell_{x_4,y_4}$ be 4 lines in $S_{\rm dual}$, and assume that $\ell_{x_1,y_1}$ and $\ell_{x_2,y_2}$ intersect at $(a_{1,2},b_{1,2})$, and that $\ell_{x_3,y_3}$ and $\ell_{x_4,y_4}$ intersect at $(a_{3,4},b_{3,4})$. Moreover, assume that these two intersection points are different. We can write the coordinates of these intersection points as
$$a_{1,2}=\frac{y_1-y_2}{x_1-x_2}, \qquad b_{1,2}=y_1-x_1\cdot\frac{y_1-y_2}{x_1-x_1},$$
$$a_{3,4}=\frac{y_3-y_4}{x_3-x_4}, \qquad b_{3,4}=y_3-x_3\cdot\frac{y_3-y_4}{x_3-x_4}.$$
Now if $a_{1,2}\neq a_{3,4}$, then
\begin{eqnarray*}
\|(a_{1,2},b_{1,2})-(a_{3,4},b_{3,4})\|_2 &\geq&
\left|a_{1,2}-a_{3,4}\right|=
\left|\frac{y_1-y_2}{x_1-x_2} - \frac{y_3-y_4}{x_3-x_4} \right|\\
&=& \left| \frac{(y_1-y_2)(x_3-x_4)-(y_3-y_4)(x_1-x_2)}{(x_1-x_2)(x_3-x_4)} \right|\geq\frac{1}{d^2},
\end{eqnarray*}
and if $a_{1,2}=a_{3,4}$, then 
$$\|(a_{1,2},b_{1,2})-(a_{3,4},b_{3,4})\|_2\geq|b_{1,2}-b_{3,4}|=|y_1-y_3- a_{1,2}(x_1-x_3)|\geq\frac{1}{d}.$$
\end{proof}

The following lemma states the utility guarantees of \texttt{SelectHalfplane}.
\begin{lemma}\label{lem:SelectHalfplane}
Consider the execution of \texttt{SelectHalfplane} on a sample $S$, and assume that there exists a hypothesis $f\in\halfplane_d$ with $q(f)\geq \lambda$. Then the probability that \texttt{SelectHalfplane} outputs a hypothesis $f'$ with $q(f')<\lambda-\frac{8}{\eps}\ln(\frac{2d}{\beta})$ is at most $\beta$.
\end{lemma}

\begin{proof}
Denote $F=[-2d^2,6d^2]$ and $D=[-2d^2,2d^2]$.
Let $\hat{R}=\{r^1_1,r^2_1,\dots,r^1_{|R|},r^2_{|R|}\}$ denote the regions defined in Step~3, and let $B\subseteq \hat{R}$ denote the subset of all regions s.t.\ the halfplanes that correspond to points in these regions have quality less than $\lambda-\frac{8}{\eps}\ln(\frac{2d}{\beta})$. Then the probability that \texttt{SelectHalfplane} outputs a hypothesis $f'$ with $q(f')<\lambda-\frac{8}{\eps}\ln(\frac{2d}{\beta})$ is at most
\begin{eqnarray*}
\sum_{r\in B}\int_r{p(a,b)}\;{\rm d}^2(a,b) 
&\leq& \sum_{r\in B}\int_r{\frac{\exp(\eps \lambda-8\ln(\frac{2d}{\beta}))}{N}}\;{\rm d}^2(a,b)\\
&\leq& \frac{\exp(\eps \lambda-8\ln(\frac{2d}{\beta}))\cdot{\rm area}\left(F\times D\right)}{N} = \frac{\exp(\eps\lambda-8\ln(\frac{2d}{\beta}))\cdot 32 d^4}{N}\\
&\leq& \frac{\exp(\eps \lambda-8\ln(\frac{2d}{\beta}))\cdot 32 d^4}{1/(4d^4)\cdot\exp(\eps \lambda)}
=128 d^8 \exp(- 8\ln(2d/\beta)) \leq\beta.
\end{eqnarray*}
\end{proof}

Combining Lemma~\ref{lem:SelectHalfplane} with Claims~\ref{claim:smallEmpiricalError} and~\ref{claim:SelectHalfplanePrivacy} yields our private learners for convex polygons:

\begin{theorem}
There exists an efficient $(\eps,\delta)$-differentially private $(\alpha,\beta)$-PAC learner for $\ckgon$ over examples from $X_d^2$ with sample complexity 
$$O\left(\frac{k}{\alpha\epsilon}   \log\left(\frac{1}{\alpha}\right)   \log\left(\frac{1}{\delta}\right) \log\left(\frac{dk}{\beta}\log\frac{1}{\alpha}\right)   \right).$$
\end{theorem}

\section{Extension to Union of Non-Convex Polygons}
In this section we briefly describe how our techniques from the previous sections can be used to learn the class of (simple) polygons in the plane, as defined next. For a simple and closed curve\footnote{A curve is simple and closed if it does not cross itself and ends at the same point where it begins.} $C$, we use $\interior(C)$ to denote the union of $C$ and its bounded area. We define the class of all polygons in the plane with (at most) $k$ edges as
$$
\kgon=\left\{\interior(C): \begin{array}{l}
	C \text{ is a simple and closed curve in } \R^2,\\
	\text{consisting of at most } k \text{ line segments}
\end{array}\right\}.
$$

By standard arguments in computational geometry, every such polygon with $k$ edges can be represented as the union of at most $k$ triangles, each of which can be represented as the intersection of at most 3 halfplanes (since a triangle is a convex polygon with 3 edges). Let us denote the class of all triangles in the plane as $\triang$. Hence,
$$\kgon\subseteq\triang^{\vee k}\subseteq\left(\halfplane^{\wedge3}\right)^{\vee k}.$$
Thus, in order to learn polygons with $k$ edges, it suffices to construct a learner for the class $\left(\halfplane^{\wedge3}\right)^{\vee k}$. In fact, this class captures {\em unions} of polygons with a total of at most $k$ edges.
In addition, similar arguments to those given in Section~\ref{sec:convexPolygons} show that if input examples come from $X_d^2=\{0,1,\dots,d\}^2$, then it suffices to construct a learner for $\left(\halfplane_d^{\wedge3}\right)^{\vee k}$, which we can do using our techniques from Sections~\ref{sec:GenericConstruction} and~\ref{sec:convexPolygons}. 

First, as we mentioned, a straightforward modification to algorithm \texttt{SetCoverLearner} yields an algorithm for learning classes of the form $C\subseteq H^{\vee k}$ (instead of $C\subseteq H^{\wedge k}$ as stated in Section~\ref{sec:GenericConstruction}). Now, to get an efficient construction, we need to specify the selection procedure for choosing a hypothesis $h_j\in\halfplane_d^{\wedge3}$ in each step of \texttt{SetCoverLearner}. As before, given an input sample $S$ we consider the dual plane $D^2$ s.t.\ every input example in $S$ from the primal plane corresponds to a line in the dual plane, and every point from the dual plane corresponds to a halfplane in the primal plane. Recall that in the previous section we identified a hypothesis (which was a halfplane) with a point in the dual plane. The modification is that now a hypothesis is a triangle which we identify with {\em three points} in the dual plane (these 3 points correspond to 3 halfplanes in the primal plane, whose intersection is a triangle). Our modified selection procedure is presented as algorithm \texttt{SelectTriangle}. 

We use $\kuniongon$ to denote the class of all unions of (simple) polygons with a total of at most $k$ edges. That is, every hypothesis $h\in\kuniongon$ can be written as $h=h_1\vee\dots\vee h_m$ for $(h_1,\dots,h_m)\in\left(\kigon\times\dots\times\kmgon\right)$ where $k_1+\dots+k_m\leq k$. A similar analysis to that of the previous section shows the following result.

\begin{theorem}
There exists an efficient $(\eps,\delta)$-differentially private $(\alpha,\beta)$-PAC learner for $\kuniongon$ over examples from $X_d^2$ with sample complexity 
$$O\left(\frac{k}{\alpha\epsilon}   \log\left(\frac{1}{\alpha}\right)   \log\left(\frac{1}{\delta}\right) \log\left(\frac{dk}{\beta}\log\frac{1}{\alpha}\right)   \right).$$
\end{theorem}

\begin{algorithm*}[t]

\caption{\bf\texttt{SelectTriangle}}\label{alg:SelectTriangle}

\noindent {\bf Input:} Labeled sample $S=\{((x_i,y_i),\sigma_i)\}_{i=1}^n\in(X^2_d\times\{\pm1\})^n$, privacy parameter $\eps$, quality function $q:\halfplane_d^{\wedge3}\rightarrow\R$.

\begin{enumerate}[leftmargin=15pt,rightmargin=10pt,itemsep=1pt]

\item Denote $D=\left[-2d^2,2d^2\right]$ and $F=\left[-2d^2,6d^2\right]$. We will refer to the axes of $D^2$ and of $F\times D$ as $a$ and $b$.
\item Identify every example $((x,y),\sigma)\in S$ with the line $\ell_{x,y}$ in $D^2$ defined by the equation $y=xa+b$, where $a,b$ are the variables and $x,y$ are the coefficients. Denote $S_{\rm dual}=\{\ell_{x,y} : ((x,y),\sigma)\in S\}$.
\item Let $R=\{r^1_1,r^1_2,\dots,r^1_{|R|}\}$ denote the partition of $D^2$ into regions defined by the lines in $S_{\rm dual}$. Also let $R'=\{r^2_1,\dots,r^2_{|R|}\}$ be a partition of $\left[2d^2,6d^2\right]\times D$ identical to $R$ except that it is shifted by $4d^2$ on the $a$ axis. Denote $\hat{R}=R\cup R'$.
\item For every $1\leq i\leq |R|$, let $w_i$ denote the area of region $r^1_i$ (which is the same as the area of $r^2_i$), and let $(a^1_i,b^1_i)\in r^1_i$ and $(a^2_i,b^2_i)\in r^2_i$ be arbitrary points in these regions.
\item Denote $N=\sum_{ r^{j_1}_{i_1},r^{j_2}_{i_2},r^{j_3}_{i_3}\in\hat{R} }w_{i_1}\cdot w_{i_2}\cdot w_{i_3}\cdot\exp(\eps\cdot q(f_{a^{j_1}_{i_1},b^{j_1}_{i_1}}\wedge f_{a^{j_2}_{i_2},b^{j_2}_{i_2}}\wedge f_{a^{j_3}_{i_3},b^{j_3}_{i_3}}))$, where $f_{a^{j_\ell}_{i_\ell},b^{j_\ell}_{i_\ell}}$ is a halfplane in $\halfplane_d$. 
\item Choose and return a random tuple $(a_1,b_1,a_2,b_2,a_3,b_3)\in (F\times D)^3$ with probability density function $p:(F\times D)^3\rightarrow\R$ defined as $p(a_1,b_1,a_2,b_2,a_3,b_3)=\frac{1}{N}\cdot\exp(\eps\cdot q(f_{a_1,b_1}\wedge f_{a_2,b_2}\wedge f_{a_3,b_3}))$.
\end{enumerate}
\end{algorithm*}

\section{Conclusion and Future Work}

In this work we presented a computationally efficient differentially private PAC learner for simple geometric concepts in the plane, which can be described as the union of polygons. Our results extend to higher dimensions by replacing lines with hyperplanes, and triangles with simplices. The running time, however, depends exponentially on the dimension. Our results also extend, via linearization, to other simple geometric concepts whose boundaries are defined by low degree polynomials, such as balls. In general, the dimension of the linearization depends on the degrees of the polynomials. This motivates the open problem of improving the dependency of the running time on the dimension of the problem. 


\bibliographystyle{abbrv}

\end{document}